\documentclass[twoside]{article}
\usepackage{palatino,epsfig,latexsym,natbib}

\usepackage[english]{babel}
\usepackage{amsthm, amsmath, amsxtra, amsfonts, amssymb, amstext}
\usepackage{booktabs}
\usepackage{nicefrac}
\usepackage{xspace}
\usepackage{url}
\usepackage{graphics,color}
\usepackage[colorlinks]{hyperref}
\definecolor{linkblue}{rgb}{0.1,0.1,0.8}
\hypersetup{colorlinks=true,linkcolor=linkblue,filecolor=linkblue,urlcolor=linkblue,citecolor=linkblue}
\usepackage[algo2e,ruled,vlined,linesnumbered]{algorithm2e}
\usepackage{multirow}
\newcommand{\assign}{\leftarrow}

\usepackage{array}
\usepackage{booktabs}
\usepackage{bigdelim}

\newtheorem{theorem}{Theorem}
\newtheorem{lemma}[theorem]{Lemma}

\newtheorem{remark}[theorem]{Remark}

\newcommand{\R}{\mathbb{R}}

\newcommand{\A}{\mathcal{A}}
\newcommand{\C}{\mathcal{C}}
\newcommand{\F}{\mathcal{F}}
\newcommand{\I}{\mathcal{I}}
\newcommand{\calE}{\mathcal{E}}

\renewcommand{\epsilon}{\varepsilon}
\newcommand{\eps}{\varepsilon}

\DeclareMathOperator{\E}{E}

\DeclareMathOperator{\dist}{dist}

\newcommand{\oea}{$(1 + 1)$~EA\xspace}

\newcommand{\ga}{$(1 + (\lambda,\lambda))$~GA\xspace}

\newcommand{\onemax}{\textsc{OneMax}\xspace}
\newcommand{\OM}{\textsc{Om}\xspace}
\newcommand{\OMP}{\textsc{Om}\xspace}
\newcommand{\xopt}{\ensuremath{x_{\text{opt}}}}

\newcommand{\jump}{\textsc{Jump}\xspace}

\begin{document}

\title{Introducing Elitist Black-Box Models: When Does Elitist Selection Weaken the Performance of Evolutionary Algorithms?}
\author{\bf{Carola Doerr}\footnote{CNRS and Sorbonne Universit\'es, UPMC Univ Paris 06, CNRS, LIP6 UMR 7606, 4 place Jussieu, 75005 Paris, France}, \qquad\bf{Johannes Lengler}\footnote{ETH Z\"urich, Institute for Theoretical Computer Science, 8092 Z\"urich, Switzerland}}
\maketitle

%

\maketitle

\begin{abstract}

Black-box complexity theory provides lower bounds for the runtime of black-box optimizers like evolutionary algorithms and serves as an inspiration for the design of new genetic algorithms. Several black-box models covering different classes of algorithms exist, each highlighting a different aspect of the algorithms under considerations. In this work we add to the existing black-box notions a new \emph{elitist black-box model}, in which algorithms are required to base all decisions solely on (a fixed number of) the best search points sampled so far. Our model combines features of the ranking-based and the memory-restricted black-box models with elitist selection. 

We provide several examples for which the elitist black-box complexity is exponentially larger than that the respective complexities in all previous black-box models, thus showing that the elitist black-box complexity can be much closer to the runtime of typical evolutionary algorithms. 

We also introduce the concept of $p$-Monte Carlo black-box complexity, which measures the time it takes to optimize a problem with failure probability at most $p$. Even for small~$p$, the $p$-Monte Carlo black-box complexity of a function class $\F$ can be smaller by an exponential factor than its typically regarded Las Vegas complexity (which measures the \emph{expected} time it takes to optimize $\F$).

\end{abstract}


\sloppy{
\section{Introduction}
\label{sec:Intro}

Black-box models are classes of algorithms that are designed to help us understand how efficient commonly used search strategies like evolutionary algorithms (EAs) and other randomized search heuristics (RSHs) are. 
Several models exist, each designed to analyze a different aspect of search heuristics. For example, the memory-restricted model~\citep{DrosteJW06,DoerrW14memory} helps us to understand the influence of the population size on the efficiency of the search strategy, while the ranking-based black-box model~\citep{DoerrW14ranking,TeytaudG06,FournierT11} analyzes how much a heuristic loses by not using absolute but merely relative fitness values.

Having been introduced to the evolutionary computation community in~\citep{DrosteJTW03, DrosteJW06}, black-box complexity is a young but highly active area of current research efforts~\citep{TeytaudG06,AnilW09,FournierT11,ABB,DoerrJKLWW11,LehreW12,DoerrKLW13,DoerrW14memory,DoerrW14ranking,DoerrW14arity,DoerrDK14jump,DoerrDK14,ParallelBBC14,ParallelBBC15,Jansen15}. 
The insights from black-box complexity studies can be used to design more efficient genetic algorithms, as the recent \ga from~\citep{DoerrDE15} shows. 

We contribute to the existing literature a new model, which we call the \emph{elitist black-box model}. As the name suggests, our model is designed to analyze the effect of elitist selection rules on the performance of search heuristics. We do so by enforcing that the algorithms in this model maintain a population that contains only the $\mu$ best so-far sampled individuals. Here, $\mu$ is a parameter of the model (called the \emph{memory-} or \emph{population-size}), while quality is measured according to increasing fitness values. Note that for population size $\mu >1$, the $\mu$ best so-far sampled search points may have different fitness values. On the other hand, if more than $\mu$ search points of current-best fitness have been sampled, only $\mu$ of them can be stored in the population. 

In the evolutionary computation (EC) community elitist selection is very common as it can be seen as a literate interpretation of the ``survival of the fittest'' principle in an optimization context. Unfortunately, this term is not used consistently in the EC literature. As described above, we call an algorithm elitist if and only if the next generation \emph{consists} of the $\mu$ best so-far search points. This is a commonly used convention in the theory of EA subcommunity. In contrast, other subcommunities call an algorithm elitist if and only if the next generation \emph{contains} one of the best so-far solutions. Another notion defines an algorithm to be elitist if and only if \emph{every} best so-far search point enters the new population, so the next population must be larger than $\mu$ if there are more than $\mu$ search points of current-best fitness. Finally, yet another notion of elitist requires that the new population \emph{only} consists of search point of the current-best fitness value, thus the next population must be smaller than $\mu$ if there are less than $\mu$ best so-far search points. Researchers using such other notions would therefore rather call our elitist black-box model a \emph{black-box model with truncation selection}. 

A short version of this work has been presented at the GECCO conference 2015 in Madrid, Spain~\citep{DoerrL15Model}.

\subsection{Previous Work}

In contrast to classical complexity notions discussed in the computer science literature, black-box complexity focuses on so-called black-box optimizers, i.e., algorithms that do not have access to the function at hand other than by evaluating possible solutions, which are referred to as \emph{search points}. 
If $\C$ is a class of black-box optimizers, then the $\C$-black-box complexity of a class $\F$ of functions is measured by the expected number of function evaluations that are needed by a best possible algorithm in $\C$ to optimize any instance $f \in \F$ (formal definitions will be given in Section~\ref{sec:model}). This number is a lower bound for the efficiency of any algorithm $\A \in \C$ and thus helps us understand how adequate an algorithmic choice is for the given problem class $\F$. 

Among the most important algorithmic choices in the design of evolutionary algorithms are the population size, the sampling strategies (often called variation operators), and the selection rules.
Existing black-box models cover these aspects in the following way. 
While the \emph{memory-restricted} model~\citep{DrosteJW06,DoerrW14memory} and the \emph{parallel} black-box model~\citep{ParallelBBC14,ParallelBBC15} analyze the influence of the population-size, 
the \emph{unbiased} model~\citep{LehreW12,ABB,DoerrJKLWW11,DoerrKLW13} considers the efficiency of search strategies using only so-called unbiased variation operators. The influence of the selection rules have been analyzed in the \emph{comparison-based} and \emph{ranking-based} black-box model~\citep{TeytaudG06,FournierT11,DoerrW14ranking}, with a focus on not revealing full fitness information to the algorithm but rather the comparison or the ranking of search points. 
The idea behind these models is that, in contrast to other search strategies like the physics-inspired simulated annealing, many evolutionary algorithms base their selection solely on \emph{relative} and not on \emph{absolute} fitness values. By providing only relative fitness values, the models aim at understanding how this worsens the performance of the algorithms, and indeed it can be shown that for some function classes the ranking-based and the comparison-based black-box complexities are larger than the unrestricted ones.

While the comparison-based and the ranking-based models provide only relative fitness values, they do not require the algorithms to always select the better ones. Search strategies that adhere to this selection rule are called \emph{elitist} algorithms in the evolutionary computation literature. Many common and widely applied black-box optimization strategies like $(\mu +\lambda)$ EAs as well as local hill climbers such as Randomized Local Search (RLS) are of this type. On the other hand, many practical algorithms intentionally keep suboptimal solutions to enhance population diversity, or to better explore the search space~\citep{Ursem02, CrepinsekLM13}. It has been shown that in some situations, specific elitist algorithms like RLS or the $(\mu+1)$ EA are inferior to non-elitist algorithms~\citep{FriedrichOSW09,JagerskupperS07,OlivetoZ15}. In this paper, we go one step further and investigate the performance of \emph{all} elitist algorithms simultaneously.

As mentioned before, algorithms in the ranking-based and comparison-based models do not need to be elitist. \mbox{(Ab-)Using} this, algorithms can be designed in both models that have much smaller runtimes than typical EAs~\citep{DoerrW14ranking, DoerrW14memory}. 
We will see that such algorithms can crucially profit from eventually giving preference to search points of fitness inferior to that of the current best search points. 

\subsection{Our Model, New Complexity Measures, and Results}
\label{sec:ourmodel}

We provide in this work a model to analyze the impact of elitist selection on the runtime of black-box optimizers. In this \emph{elitist black-box model} the population of the algorithms may contain only search points of best-so-far fitness values. That is, if the population size is $\mu$, then at any point in time only the $\mu$ best-so-far search points (of possibly different fitness values) are allowed to be kept in the population. Ties may be broken arbitrarily. For example, if more than $\mu$ search points of current-best fitness have been sampled, only (an arbitrary selection of) $\mu$ of them can be stored in the population. All other previously sampled search points are not allowed to influence the behavior of the algorithm any more. 

We show (Section~\ref{sec:elitistlarge}) that already for quite simple function classes there can be an exponential gap between the efficiency of elitist and non-elitist black-box algorithms. As we shall see in Section~\ref{sec:elitistlarge} this even remains true if we regard (1+1) memory-restricted unary unbiased comparison-based algorithms, which constitutes the most restrictive combination of the existing black-box models. We will see that such algorithms can crucially profit from eventually giving preference to search points of fitness inferior to that of the current best search points. 
We also show (Section~\ref{sec:jump}) that some shortcomings of previous models can be eliminated when they are combined with an elitist selection requirement. More precisely we show that the elitist unary unbiased black-box complexity of $\jump_k$ is of order $\Omega(\binom{n}{k+1})$ and thus non-polynomial for $k=\omega(1)$. In contrast, the unary unbiased black-box complexity of $\jump_k$ is known to be polynomial even for extreme values of~$k$~\citep{DoerrDK14jump}.

In previous models, the black-box complexity has been defined in a \emph{Las Vegas} manner, that is, it measures the expected number of function evaluations until the algorithm hits the optimum. On the other hand, many results in the black-box complexity literature are based on algorithms that with high (or constant) probability find the optimum after a certain number of steps, and then \emph{random restarts} are used to bound the expected runtime.
In (the strict version of) the elitist model, algorithms are not allowed to do random restarts since new search points can be kept in the population only if they are among the $\mu$ best ones sampled so far. Since this is a rather artificial problem (many real-world optimization routines make use of restarts), we introduce in this work the concept of \emph{Monte Carlo black-box complexities}. Roughly speaking, the $p$-Monte Carlo black-box runtime of a black-box algorithm $A$ on a function $f$ is the minimal number of queries $A$ needs in order to find the optimum of $f$ with probability at least $1-p$. The complexity class is then derived in the usual way, cf.\! Section~\ref{sec:lasvegas}. We regard in our work both Monte Carlo complexities and standard (i.e., \emph{Las Vegas}) complexities. For elitist black-box algorithms these two notions can differ substantially as we shall see in Section~\ref{sec:onemaxdouble}. 

In the following we consider only discrete search spaces, and even more restrictively, only pseudo-Boolean functions $f:\{0,1\}^n \rightarrow \R$. However, generalizations to non-finite or continuous search spaces are straightforward.

\section{The Elitist Black-Box Model}
\label{sec:model}

The \emph{elitist black-box model} covers all algorithms that follow the pseudo-code in Algorithm~\ref{alg:elitist}. 
To describe it in a more detailed fashion, a $(\mu+\lambda)$ \emph{elitist black-box algorithm} is initialized by sampling $\mu$ search points. We allow these search points to be sampled \emph{adaptively}, that is, the $i$-th sample may depend on the ranking of the first $i-1$ search points, where, obviously, by ranking we regard the ranking induced by the fitness function $f$.\footnote{Two search points have the same rank if and only if they have the same fitness. The search points of $X$ with maximal $f$-values are rank one, the ones with second largest $f$-values are rank two etc.}  
In each subsequent round a $(\mu+\lambda)$ elitist black-box algorithm samples $\lambda$ new search points from distributions that depend only on the current population $X$ and the ranking of $X$. Note that in such an optimization step the offsprings do not need to be independent of each other. Assume, for example, that we create an offspring $x$ by random crossover, i.e., we take some parents from the current population and set the entries of $x$ by choosing (in an arbitrary way) some bit values from these parents; then it is allowed to also create another offspring $y$ from these parents whose entries $y_i$ in those positions $i$ in which the parents do not agree equal $1-x_i$. These two offsprings are obviously not independent of each other.
However, we do require that the offsprings are created \emph{before} any evaluation of the offsprings happens. That is, the $k$-th offspring may \emph{not} depend on the ranking or fitness of the first $k-1$ offsprings. (We have decided for this version as we feel that it best captures the spirit of EAs such as the $(\mu+\lambda)$ EA that can process the $\lambda$ offsprings in parallel.) 
When all $\lambda$ search points have been generated, the algorithm proceeds to the selection step. In this step, the algorithm sorts the $\mu+\lambda$ search points according to their fitness ranking, where it is free to break ties in any way. Then the new population consists of the $\mu$ best search points according to this ordering.

\begin{algorithm2e}
 \textbf{Initialization:} \\
 \Indp
 $X \assign \emptyset$\;
 \For{$i=1,\ldots,\mu$}{
 Depending only on the multiset $X$ and the ranking $\rho(X,f)$ of $X$ induced by $f$, choose a probability distribution $p^{(i)}$ over $\{0,1\}^n$ and sample $x^{(i)}$ according to $p^{(i)}$\;
 $X \assign X \cup \{ x^{(i)}\}$\;
 }
 \Indm
 \textbf{Optimization:}	
 \For{$t=1,2,3,\ldots$}{
 		Depending only on the multiset $X$ and the ranking $\rho(X,f)$ of $X$ induced by $f$
	\label{line:mut}	choose a probability distribution $p^{(t)}$ on $(\{0,1\}^n)_{i=1}^{\lambda}$ and 
		sample $(y^{(1)},\ldots,y^{(\lambda)}) $ according to $p^{(t)}$\;
		Set $X \assign X \cup \{y^{(1)},\ldots,y^{(\lambda)}\}$\;
  \lFor{$i=1,\ldots, \lambda$}{
  	\label{line:selection} Select $x \in \arg\min X$ and update $X \assign X \setminus \{x\}$\;}
	 }
 \caption{The $(\mu+\lambda)$ elitist black-box algorithm for maximizing an unknown function $f:\{0,1\}^n \rightarrow \R$}
\label{alg:elitist}
\end{algorithm2e}

The elitist black-box model covers many common EAs such as $(\mu+\lambda)$~EAs, Randomized Local Search (RLS), and other hill climbers. It does not cover algorithms with non-elitist selection rules like tournament or fitness-proportional selection. 

Several extensions and variants of the model are possible, including in particular one in which the $\mu$ first search points cannot be sampled adaptively, where the selection has to be unbiased among search points of the same rank, where only offsprings can be selected (comma strategies), or a non-ranking-based version in which absolute instead of relative fitness information is provided. Note that the latter would allow for fitness-dependent mutation rates, which are excluded by the variant analyzed here. The lower bounds presented in Sections~\ref{sec:elitistlarge} and~\ref{sec:jump} actually hold for this non-ranking-based model (and are thus even more powerful than such only applicable to the model described in Algorithm~\ref{alg:elitist}).
The model can certainly also be extended to an unbiased elitist one, in which the distribution $p^{(t)}$ in line~\ref{line:mut} of Algorithm~\ref{alg:elitist} has to be unbiased in the sense of Lehre and Witt~\citep{LehreW12}. See Section~\ref{sec:jump} for results on the unbiased elitist model.

Note that elitist black-box algorithms covered by Algorithm~\ref{alg:elitist} are \emph{memory-restricted} in the sense of~\citep{DrosteJW06, DoerrW14memory}, that is, they cannot store any other information than the current population and its ranking. All information about previous search points (e.g., their number) has to be discarded. 
The (1+1) version of the elitist model is \emph{comparison-based} (that is, a query reveals only if an offspring has worse, equal, or better fitness than its parent), while the $(\mu+\lambda)$ versions are \emph{ranking-based} in the sense of~\citep{DoerrW14ranking}. This means that the algorithm has no information about absolute fitness values, but it knows how the fitness values of the $(\mu+\lambda)$ search points compare to each other. 
To stress the difference between the latter two models, we remark the following: For $\mu>1$ or $\lambda>1$ the ranking-based black-box model provides more information than the comparison-based one as it gives a full ranking of all current search points, while in the comparison-based we always have to select two search points which are compared against each other. The ranking-based black-box complexity can thus be smaller by a logarithmic factor than the comparison-based complexity.\footnote{Binary search shows that the gap between the two notions cannot be larger than a logarithmic factor, but on the other hand this logarithmic gap occurs already for the simple \onemax problem. While it's ranking-based black-box complexity is only $\Theta(n/\! \log n)$~\citep{DoerrW14ranking}, its comparison-based complexity is $\Theta(n)$ by a straightforward application of Yao's Principle.} However, in the memory-restricted black-box models (with constant $\mu$ and $\lambda$), the ranking-based and the comparison-based black-box complexities are asymptotically equal.

\subsection{Monte Carlo vs. Las Vegas Black-Box Complexities}
\label{sec:lasvegas}

As discussed in Section~\ref{sec:ourmodel}, usually the black-box complexity of a function class $\F$ is defined in a \emph{Las Vegas} manner (measuring the \emph{expected number} of function evaluations), while in the case of elitist black-box complexity we also introduce a $p$-\emph{Monte Carlo black-box complexity}, where we allow some failure probability $p$ (see below for formal definitions). If we make a statement about the Monte Carlo complexity without specifying $p$, then we mean that for \emph{every constant} $p>0$ the statement holds for the $p$-Monte Carlo complexity. However, we sometimes also regard $p$-Monte Carlo complexities for non-constant $p=p(n) = o(1)$, thus yielding high probability statements.


For most black-box complexities, the Las Vegas and the Monte Carlo notions are closely related: every Las Vegas algorithm is also (up to a factor of $1/p$ in the runtime) a $p$-Monte Carlo algorithm by Markov's inequality, and a Monte Carlo algorithm can be turned into a Las Vegas algorithm by restarting the algorithm until the optimum is found. 
In particular, if restarts are allowed then Las Vegas and Monte Carlo complexities differ by at most a constant factor. This has been made explicit in~\citep[Remark 2]{DoerrDK14} and is heavily used there as well as in a number of other results on black-box complexity.
It is not difficult to see that such a reasoning fails for elitist black-box algorithms, as they are not allowed to do arbitrary restarts: if the sampled solution intended for a restart is not as good as the ones currently in the memory, it has to be discarded (line~\ref{line:selection} of Algorithm~\ref{alg:elitist}). Las Vegas and Monte Carlo elitist black-box complexities may therefore differ significantly from each other, see Section~\ref{sec:onemaxdouble} for an example with exponentially large gap. 


We come to the formal definition. Let $\F$ be a class of pseudo-Boolean functions, and let $p\in [0,1)$. 
The \emph{Las Vegas complexity} of an algorithm $A$ for $\F$ is the maximum expected number of function evaluations of $f$ before $A$ evaluates an optimal search point for the first time, where the maximum is taken over all $f \in \F$. 
The Las Vegas complexity of $\F$ with respect to a class $\A$ of algorithms is the minimum (``best'') Las Vegas complexity among all $A \in \A$ for $\F$. 
The $p$-Monte Carlo complexity of $\F$ with respect to $\A$ is the minimum number $T$ such that there is an algorithm in $\A$ which has for all $f\in\F$ a probability of at least $1-p$ to find an optimum within the first $T$ function evaluations. 
The \emph{$(\mu + \lambda)$ elitist Las Vegas (elitist $p$-Monte Carlo) black-box complexity} of $\F$ is the Las Vegas ($p$-Monte Carlo) complexity of $\F$ with respect to the class of all $(\mu+\lambda)$ elitist black-box algorithms. 

To ease terminology, we will say that an algorithm \emph{spends time $t$} on a function $f$ if it uses at most $t$ function evaluations on $f$. Moreover, we call the \emph{runtime} of an algorithm $A$ on a function $f$ the random variable describing the number of function evaluations of $A$ until it evaluates for the first time an optimal search point of $f$. In this way, the Las Vegas complexity of $\A$ on $\F$ is the worst-case (over all $f\in F$) expected runtime of the best algorithm $A\in\A$. 

If we are interested in the asymptotic $p$-Monte Carlo complexity of an algorithm $A$ on a function class $\F$ then we will frequently make use of the following observation, which follows from Markov's inequality and the law of total expectation. 
\begin{remark}
\label{rem:1}
Let $p \in (0,1)$. Assume that there is an event $\calE$ of probability $p_{\calE} < p$ such that conditioned on $\neg \calE$ the algorithm $A$ finds the optimum after \emph{expected} time at most $T$. Then the $p$-Monte Carlo complexity of $A$ on $f$ is at most $(1-p_\calE)(p-p_\calE)^{-1}T$. In particular, if $p-p_\calE = \Omega(1)$ then the $p$-Monte Carlo complexity is $O(T)$.
\end{remark}

\begin{proof}
Let $R$ be the expected runtime of $A$ on $f$. 
By Markov's inequality and the law of total expectation we have 
\begin{align*}
\Pr\left[R \geq \frac{1-p_\calE}{p-p_\calE}\cdot T\right]
& \leq 
\Pr\left[R \geq \frac{1-p_\calE}{p-p_\calE}\cdot T \,\middle|\, \neg\calE\right] \cdot\Pr\left[\neg\calE\right] + \Pr\left[\calE\right]\\
& \leq 
\frac{p-p_\calE}{1-p_\calE} (1-p_\calE) + p_\calE
= 
p \hfill.
\end{align*}
~\vspace{-6ex}
\\~\hspace*{\fill}\end{proof}

\subsection{(Non-)Applicability of Yao's Principle}
\label{sec:Yao}

A convenient tool in black-box complexity theory is Yao's Principle. In simple words, Yao's Principle allows to restrict one's attention to bounding the expected runtime $T$ of a best-possible \emph{deterministic} algorithms on a \emph{random} input instead of regarding the best-possible performance of a \emph{random} algorithm on an \emph{arbitrary} input. Analyzing the former is often considerably easier than directly bounding the performance of any possible randomized algorithm. Yao's Principle states that $T$ is a lower bound for the expected performance of a best possible randomized algorithm for the regarded problem. In most applications a very easy distribution on the input can be chosen, often the uniform one.  Formally, Yao's Principle is the following.  

\begin{lemma}[Yao's Principle~\citep{Yao77}]
\label{lem:Yao}
Let $\Pi$ be a problem with a finite set $\mathcal I$ of input instances (of a fixed size) permitting a finite set $\A$ of deterministic algorithms. Let $p$ be a probability distribution over $\mathcal I$ and $q$ be a probability distribution over $\A$. Then, 
\begin{align}\label{eq:Yao}
	\min_{A \in \A} \E[T(I_p, A)] \leq \max_{I \in \mathcal I} \E[T(I,A_q)]\, , 
\end{align}
where $I_p$ denotes a random input chosen from $\mathcal I$ according to $p$, $A_q$ a random algorithm chosen from $\C$ according to $q$ and $T(I,A)$ denotes the runtime of algorithm $A$ on input $I$. 
\end{lemma}

It is interesting to note that the informal interpretation of Yao's Principle given above does not apply to elitist algorithms. 
To illustrate this phenomenon, let us consider the (1+1) elitist model, though the argument can be easily extended to population-based elitist algorithms. Let $p$ be the uniform distribution over the instances 
\begin{align}
\label{def:OMz}
\OM_z:\{0,1\}^n \rightarrow \R, x \mapsto n-\sum_{i=1}^n{(x_i \oplus z_i)},
\end{align}
$z \in \{0,1\}^n$, of the well-known \onemax problem (see Section~\ref{sec:elitistlarge} for some background on this problem).
Let $A$ be any deterministic algorithm. 
Then we will show that $A$ has a positive (in fact, fairly large) probability during the optimization of $I_p$ of getting stuck in some search point $x$: as a deterministic (1+1) elitist algorithm there exists a search point~$y=y(x)$ such that whenever the algorithm sees $x$ in the memory it samples $y$ next. If the $\OM_z$-fitness of $y$ is strictly smaller than that of~$x$, offspring~$y$ has to be discarded immediately, in which case the algorithm is in exactly the same situation as before. It can thus never escape from~$x$, and the expected runtime of the algorithm on $\OM_z$ is infinite. It remains to show that this situation happens with positive probability. Assume that the first two search points that $A$ queries are $x$ and $y$. Note that $A$ does not obtain any information from querying $x$, so $y$ is independent of the fitness function. Moreover, we may assume $x \neq y$. Then there are at most $2^{n-1}$ search points $z$ such that $\OM_z(x) = \OM_z(y)$. Moreover, by symmetry (and the uniformity of $p$) half of the at least $2^{n-1}$ remaining search points satisfy $\OM_z(x) < \OM_z(y)$, so $A$ runs into an infinite loop with positive probability. Thus every deterministic (1+1) elitist algorithm has an infinite expected runtime on a uniformly chosen \onemax instance. The lower bound in (\ref{eq:Yao}) is thus infinite, too, suggesting that the elitist black-box complexity of this problem is infinite as well. However, there are simple elitist randomized search strategies that have finite expected runtime on \onemax, for example, RLS and the \oea . 

Why does this example not contradict Yao's Principle? Reading Lemma~\ref{lem:Yao} carefully, we see that it makes a statement only about such randomized algorithms that are a convex combination of deterministic ones. In other words, the randomized algorithms (on a fixed input size) are given by making one random choice at the beginning, determining which of the finitely many deterministic algorithms we apply. For typical classes of algorithms \emph{every} randomized algorithm is such a convex combination of deterministic algorithms (and randomized algorithms are, in fact, often defined this way). In this case Yao's Principle can be summarized in the way we described before Lemma~\ref{lem:Yao}, i.e., as a statement that links the worst-case expected runtime of randomized algorithms with the best expected runtime of deterministic algorithms on random input. The previous paragraph, however, explains that in the elitist black-box model there are randomized algorithms which cannot be expressed as a convex combination of deterministic ones. For this reason, we can never apply Yao's Principle directly to the class of elitist black-box algorithms. Similar considerations hold for other classes of memory-restricted black-box algorithms, but have not been mentioned explicitly in the literature. We are not aware of any other class of algorithms where such an anomaly occurs and find the putative non-applicability of Yao's Principle quite noteworthy.

Due to the problems outlined above, we will often consider in our lower bound proofs a superset $\A'$ of algorithms which contains all elitist ones and which has the property that every randomized algorithm in $\A'$ can be expressed as a convex combination of deterministic ones. A lower bound shown for this broader class trivially applies to all elitist black-box algorithms. Observe in particular that in a class of black-box algorithms where every algorithm knows the number of previous steps, every randomized strategy is a convex combination of deterministic strategies: the algorithm can just flip all coins in advance, and essentially use the $i$-th random bit (or bit string) in the $i$-th step. Hence, such classes may be used to apply Yao's Principle. Note that the lower bounds obtained in the subsequent sections are still much stronger than those for any of the previous black-box models.
%

\section{Exponential Gaps to Previous Models}
\label{sec:elitistlarge}

We provide some function classes for which the elitist black-box complexity is exponentially larger than their black-box complexities in any of the previously regarded models. In particular, the black-box complexity will still be small in a model in which all algorithms have to be unbiased, memory-restricted with size bound one, and purely comparison-based. This shows that our model strengthens the existing landscape of black-box models considerably. The example will also show that the Las Vegas complexity of a problem can be exponentially larger than its Monte Carlo complexity.

\subsection{Twin Peaks}
\label{sec:twinpeak}
We first describe a type of landscapes for which the elitist black-box complexity is exponentially large. The following theorem captures the intuition that elitist algorithms are very bad if there are several local optima that the algorithm needs to explore in order to determine the best one of them. 
This remains true if we grant the algorithm access the absolute (instead of the relative) fitness values, as we will show in Remark~\ref{rem:twinpeak}.

\begin{theorem}
\label{thm:twinpeak}
Let $\eps>0$. Let $\F$ be a class of functions from $\{0,1\}^n$ to $\R$ such that for every set $\{z_1,z_2\} \subset \{0,1\}^n$ with $z_1 \neq z_2$,
\begin{itemize}
\item there is a function $f_{z_1,z_2} \in \F$ such that $z_1$ is the unique global optimum, and $z_2$ is the unique second-best search point of $f_{z_1,z_2}$;
\item $\F$ also contains the function $f'_{z_1,z_2}$ that is obtained from $f_{z_1,z_2}$ by switching the fitness of $z_1$ and $z_2$. More formally, $f'_{z_1,z_2}$ is defined by $f'_{z_1,z_2}(z_1) = f_{z_1,z_2}(z_2)$, $f'_{z_1,z_2}(z_2) = f_{z_1,z_2}(z_1)$, and $f'_{z_1,z_2}(z) = f_{z_1,z_2}(z)$ for $z \in \{0,1\}^n\setminus\{z_1,z_2\}$.
\end{itemize}
Then the (1+1) elitist Las Vegas black-box complexity and the $(1/2-\eps)$-Monte Carlo black-box complexity of $\F$ are exponential in $n$. 
\end{theorem}

\begin{proof}[Proof of Theorem~\ref{thm:twinpeak}]
To give an intuition, we first give an outline of the proof that is not quite correct. Assume that a black-box algorithm encounters either $f_{z_1,z_2}$ or $f'_{z_1,z_2}$. By definition of $f'_{z_1,z_2}$, it does not know the global optimum before querying either $z_1$ or $z_2$. It thus needs to query either $z_1$ or $z_2$ first. Assume that it queries $z_1$ first. Then if the algorithm is unlucky (if $z_1$ is not the global optimum, i.e., the algorithm optimizes $f'_{z_1,z_2}$), the algorithm is stuck in a local optimum which it cannot leave except by sampling the optimum $z_2$. Due to the memory restriction the algorithm has lost any information about the objective function except possibly that $z_1$ is one of the two best search points. But since $f'_{z_1,z} \in \F$ for all $z \in \{0,1\}^n \setminus \{z_1\}$, the algorithm would then have lost any information about $z_2$, and would still have test $2^n-1$ possible optima.
  
Unfortunately, this intuitive argument fails: after querying $z_1$ the algorithm does have some information about $z_2$, despite the severely restricted memory. For example, consider an algorithm $A_{z_1,z_2}$ that only queries $z_1$ if either it has already tested the second-best search point, or if it has identified the function to be $f_{z_1,z_2}$ or $f'_{z_1,z_2}$. In particular, for every $z \neq z_2$, in order to optimize $f_{z_1,z}$ or $f'_{z_1,z}$ the algorithm queries $z$ before querying $z_1$. Then $A_{z_1,z_2}$ never takes $z_1$ into memory, unless either $z_1$ is the optimum or $z_2$ is the optimum. In particular, if $A_{z_1,z_2}$ sees $z_1$ in its memory, and $z_1$ is not optimal, then it can query the global optimum in the next step. So an algorithm can draw information from the order in which it queries $z_1$ and $z_2$. However, informally this information is limited to one bit. Therefore, the algorithm can not gain much, and the intuitive argument outlined at the beginning still works approximatively. 

To turn this intuition into a formal proof, we employ Yao's Principle (Lemma~\ref{lem:Yao}). As described in Section~\ref{sec:Yao}, we need to consider a larger class $\A$ of algorithms defined as follows. Assume the algorithm has to optimize $f_{z_1,z_2}$ or $f'_{z_1,z_2}$. We call the time until the algorithm queries for the first time $z_1$ or $z_2$ the ``first phase'', while we call the remaining time the ``second phase''. Since we are not too much interested in the time that the algorithm spends in the first phase, we simply give away to the algorithm the set $\{f_{z_1,z_2}, f'_{z_1,z_2}\}$. That is, during the first phase the algorithm knows everything about the objective function except which of the two points $z_1$ or $z_2$ is the optimum. We also give the algorithm access to unlimited memory throughout this phase. During this phase every randomized algorithm is a convex combination of deterministic ones. So we may use Yao's Principle, choose a probability distribution on $\F$, and restrict ourselves to an algorithm $A$ that is deterministic in the first phase. For the probability distribution on $\F$, we choose a set $\{z_1,z_2\}$ of two $n$-bit strings uniformly at random, and then we pick either $f_{z_1,z_2}$ or $f'_{z_1,z_2}$, each with probability $1/2$.

Note that in the first phase the algorithm does not gain any additional information by querying any search point $z \notin \{z_1, z_2\}$ since it can predict the fitness value of $z$ without actually querying it. We may thus assume that the first query of $A$ is either $z_1$ or $z_2$. Let $\C$ be the set of all sets $\{z_1,z_2\}$, where $z_1, z_2\in \{0,1\}^n$ and $z_1 \neq z_2$. In the first phase the algorithm $A$ essentially assigns to each set $\{z_1,z_2\} \in \C$ either $z_1$ or $z_2$. 
Let us denote the corresponding function by $h_A: \C \to \{0,1\}^n$. With probability $1/2$, $h_A(z_1,z_2)$ is the global optimum, and with probability $1/2$ it is not. 

With probability $1/2$ the algorithm enters the second phase, in which we no longer allow it to access anything but the current search point and possibly its fitness. For the sake of exposition, we first consider the case that the algorithm may not access the fitness, and describe afterwards how to change the argument otherwise. The algorithm $A$ can be randomized in this second phase. Recall that the instance is taken uniformly at random, and that $A$ samples $z_1$ whenever $z_2 \in \C_{z_1} := h_A^{-1}(z_1)$. Therefore, conditioned on seeing $z_1$, the global optimum is uniformly distributed in $\C_{z_1}$. The algorithm hence needs an expected number of $\Omega(\C_{z_1})$ additional queries to find $z_2$, and the probability to find the optimum with $\alpha |\C_{z_1}|$ additional queries is at most $\alpha$.

It remains to show that $\C_{z_1}$ is large with high probability. Let $p>0$. Note that the average size over all $z_1$ (not the expectation over all instances!) of $\C_{z_1}$ is $E:= |\C|/2^n = (2^n -1)/2$. Let $D := \{z_1 \in \{0,1\}^n \mid |\C_{z_1}| \leq pE\}$. 
Then $|h_A^{-1}(D)| \leq 2^n pE \leq p|\C|$. 
Since the random instance is chosen uniformly at random from $\F$, the set $\{z_1,z_2\} \in \C$ is also uniformly at random, and with probability at least $1-p$ an instance from $\C \setminus h_A^{-1}(D)$ is chosen, and thus $|\C_{z_1}| > pE$. Thus for every $p>0$, conditioned on entering the second phase $|\C_{z_1}| > pE = \Omega(p2^n)$ with probability at least $1-p$. Choosing somewhat arbitrarily $\alpha = p=n^{-1}$ shows that with probability at least $1/2-o(1)$ the algorithm needs at least $\Omega(2^n/n)$ steps. This concludes the proof.
\end{proof}

\begin{remark}\label{rem:twinpeak}
Theorem~\ref{thm:twinpeak} essentially also holds if we allow the algorithms to access absolute fitness values. More precisely, let $\F$ be a class of functions as in~\ref{thm:twinpeak}, and let $V(\F) := \max{\{f(z) \mid f\in F, z\in \{0,1\}^n}, z \text{ is not a global maximum of $f$} \}$ be the set of all second-best fitness values. If $V(\F)$ has subexponential size, then the (1+1) elitist Las Vegas black-box complexity and the $(1/2-\eps)$-Monte Carlo black-box complexity of $\F$ remain exponential even if the algorithms have access to the absolute fitness values.
\end{remark}
\begin{proof}
The same proof as for Theorem~\ref{thm:twinpeak} still works, only that for every $a\in V(\F)$ we let $\C_{z_1,a} := \{z_2 \in \C_{z_1} \mid f_{z_1,z_2}(z_1) = a\}$. This partitions $\C_{z_1}$ into $V(\F)$ subsets, and since $|V(\F)| = 2^{o(n)}$, on average these sets are still exponentially large. The theorem now follows in the same way as before, with the sets $\C_{z_1}$ replaced by $\C_{z_1,f(z_1)}$.
\end{proof}

\paragraph{The Double OneMax Problem}
\label{sec:onemaxdouble}

Theorem~\ref{thm:twinpeak} provides us with landscapes that are very hard for elitist algorithms. We now give a more concrete example, the class of \emph{double \onemax functions}. This class is of the type as described in Theorem~\ref{thm:twinpeak}, but at the same time it is easy for a very simple non-elitist algorithm, namely a variant of RLS using restarts, cf.~Algorithm~\ref{alg:RLSrestart}. The basis for double \onemax functions is \onemax, one of the best studied example functions in the theory of evolutionary computation. The original \onemax function simply counts the number of ones in a bitstring. Maximizing \onemax thus corresponds to finding the all-ones string. 

Search heuristics are typically invariant with respect to problem encoding, and as such they have the same expected runtime for any function from the generalized \onemax function class $\onemax:=\left\{\OM_z \mid z \in \{0,1\}^n \right\}$,
where $\OM_z$ is defined by~\eqref{def:OMz}. We call $z$, the unique global optimum of function $\OM_z$, the \emph{target string} of $\OM_z$. 

A very simple heuristic optimizing \onemax in $\Theta(n \log n)$ steps is \emph{Randomized Local Search} (RLS). Since a variant of RLS will be used in our subsequent proofs, we give its pseudo-code in Algorithm~\ref{alg:RLS}. RLS is initialized with a uniform sample $x$. In each iteration one bit position $j \in [n]$ is chosen uniformly at random. The $j$-th bit of $x$ is flipped and the fitness of the resulting search point $y$ is evaluated. The better of the two search points $x$ and $y$ is kept for future iterations (favoring the newly created individual in case of ties). As is easily verified, RLS is a unary unbiased (1+1) elitist black-box algorithm.

\begin{algorithm2e}%
	\textbf{Initialization:} Sample $x \in \{0,1\}^n$ uniformly at random and query $f(x)$\;
 \textbf{Optimization:}
\For{$t=1,2,3,\ldots$}{
Choose $j \in [n]$ uniformly at random\;
Set $y\assign x\oplus e^n_{j}$ and query $f(y)$\,; //mutation step\\
\lIf{$f(y)\geq f(x)$}{$x \assign y$\,; //selection step}
}
\caption{Randomized Local Search for maximizing~$f\colon\{0,1\}^n\to\mathbb{R}$.}
\label{alg:RLS}
\end{algorithm2e}

We are now ready to define the double onemax functions. For two different strings $z_1,z_2 \in \{0,1\}^n$, let
\begin{align*}
\OM_{z_1,z_2}(x):= 
\begin{cases}
\max\{ \OM_{z_1}(x), \OM_{z_2}(x) \}, &\text{ if } x\neq  z_1,\\
n+1, &\text{ otherwise.}
\end{cases}
\end{align*}
The unique global optimum of this function is $z_1$, and $z_2$ is the unique second best search point. For all $x \notin\{z_1,z_2\}$ the fitness $\OM_{z_1,z_2}(x)$ equals $\OM_{z_2,z_1}(x)$. Unless the algorithm queries either $z_1$ or $z_2$ it can therefore not distinguish between the two functions. We consider the class of functions $\F:=\{ \OM_{z_1,z_2} \mid  z_1,z_2 \in \{0,1\}^n, z_1 \neq z_2\}$ 
and show the following.

\begin{theorem}
\label{thm:doubleOM}
Let $\eps>0$.
The (1+1) elitist $(1/2+\eps)$-Monte Carlo black-box complexity of $\F$ 
and its unary unbiased, (1+1)-memory restricted, comparison-based black-box complexity is $O(n\log n)$, while the (1+1) elitist Las Vegas black-box complexity of $\F$ and its $(1/2-\eps)$-Monte Carlo black-box complexity are exponential in $n$ even if we allow the algorithms to access absolute fitness values.
\end{theorem}

\begin{proof}[Proof of Theorem~\ref{thm:doubleOM}]
The class $\F$ satisfies the conditions from Theorem~\ref{thm:twinpeak}, so the lower bound for the (1+1) elitist Las Vegas black-box complexity of $\F$ follows immediately. For the upper bound, consider the random local search algorithm (RLS) with random restarts as given by Algorithm~\ref{alg:RLSrestart}. This algorithm is initialized like RLS. The only difference to RLS (Algorithm~\ref{alg:RLS}) is that during the optimization process, instead of mutating the current best search point, it may restart completely by drawing a point $y$ uniformly at random from $\{0,1\}^n$ and replacing the current best solution $x$ by $y$ regardless of their fitness values. 
We show that this algorithm has expected optimization time $O(n\log n)$. 

Whenever $x \notin \{z_1,z_2\}$ then the one-bit flip has probability at least $n - f(x)$ to increase the fitness of $x$ (this can be proven by an easy case distinction whether or not $\OM_{z_1}(x) \geq \OM_{z_2}(x)$). This is at least as large as the progress probability for \onemax. Therefore, if no restart happens in, say, $5n\log n$ steps (which is true with constant probability) then with high probability RLS finds either $z_1$ or $z_2$ in this time. Note that there search space is 2-vertex transitive, i.e., there is an automorphism of the search space that maps $z_1$ to $z_2$ and vice versa. By definition of $\OM_{z_1,z_2}(x)$, the same automorphism maps $\OM_{z_1,z_2}(x)$ to $\OM_{z_2,z_1}(x)$. Hence, since RLS with restarts is an unbiased algorithm, it will reach $z_1$ before $z_2$ with probability $1/2$, and similarly vice versa. Thus, when the algorithm queries either $z_1$ or $z_2$, then it finds the global optimum with probability $1/2$. Summarizing, after each restart, the algorithm has at least a constant probability to find the global optimum in the next $5n\log n$ steps. This proves both upper bounds in Theorem~\ref{thm:doubleOM}.  

\begin{algorithm2e}[t]%
	\textbf{Initialization:} 
	Sample $x \in \{0,1\}^n$ uniformly at random and query $f(x)$\;
 \textbf{Optimization:}
\For{$t=1,2,3,\ldots$}{
With probability $1/(10 n \log n)$ sample $y \in \{0,1\}^n$ uniformly at random and replace $x$ by $y$\;
\Else{
Choose $j \in [n]$ uniformly at random\;
Set $y\assign x\oplus e^n_{j}$ and query $f(y)$\,; //mutation step\\
\lIf{$f(y)\geq f(x)$}{$x \assign y$\,; //selection step}
}}
\caption{Randomized Local Search with random restarts}
\label{alg:RLSrestart}
\end{algorithm2e}
\end{proof}

\subsection{Hidden Paths}
\label{sec:example2}

We provide another example with an exponential gap between elitist and non-elitist black-box complexities, which gives some more insight into the disadvantage of elitist algorithms. We use essentially the \onemax function, patched with a path of low fitness that leads to the global optimum. In this example, every elitist algorithm fails with high probability to find the optimum in polynomial time, since it is blind to all search points of small fitness value. Both the Monte Carlo and the Las Vegas elitist black box complexity of the problem are exponential in $n$, so that (unlike the example from Section~\ref{sec:onemaxdouble}) the problem cannot be easily mended by allowing restarts. On the other hand, there are memory-restricted, unary unbiased (but not elitist) algorithms that solve the problem efficiently.
 
For $z\in \{0,1\}^n$, let $\bar z$ be the bitwise complement of $z$, i.e., $\bar z_i = 1-z_i$ for all $i \in [n]$. Let further
$
\I_{\ell} = \{\vec i = (i_1,\ldots,i_\ell) \in [n]^{\ell} \mid i_1,\ldots, i_\ell \text{ pairwise distinct}\}.
$ 
To each $\vec i \in \I_{\ell}$ and each $z^0\in \{0,1\}^n$, we associate a path $P(z^0,\vec i) = (z^0,\ldots,z^\ell)$ of length $\ell$ as follows. For $j \in [\ell]$, let $z^j \in \{0,1\}^n$ be the search point obtained from $z^{j-1}$ by flipping the $i_j$-th bit. Note that $z^j$ differs from $z^0$ in exactly $j$ bits.
 
Now we regard the set of all \onemax functions $\OM_z$ padded with a path of length $n/4$ starting at the minimum $z^0:=z^0(z)=\bar{z}$ and leading to the unique global maximum $z^{n/4}$. 
Formally, let $\F:=\{\OMP_{z,\vec i} \mid z \in \{0,1\}^n, \vec i \in \I_{n/4} \}$, where for $\vec i \in \I_{\ell}$ and $P\big(z^0,\vec i\, \big) = (z^0,\ldots,z^\ell)$,
\begin{align*}
\OMP_{z,\vec i}(x) := 
\begin{cases}
n+\OM_{z}(x), &\text{ if } x\not\in P\big(z^0,\vec i\, \big),\\
j, &\text{ if } x= z^j \text{ for } 0 \leq j < \ell,\\
2n+1, &\text{ if } x= z^{\ell}.
\end{cases}
\end{align*}

\begin{theorem}
\label{thm:example2}
The unary unbiased (1+1) memory-restricted black-box complexity of $\F$ is $O(n^2)$, while its (1+1) Monte Carlo (and thus, also Las Vegas) elitist black-box complexity is $2^{\Omega(n)}$, also for the non-ranking-based version of the elitist model in which full (absolute) fitness information is revealed to the algorithm.
\end{theorem}

\begin{proof}
For the upper bound, we need to describe a memory-restricted unary unbiased black-box algorithm $A$ that optimizes $f\in F$ in quadratic time. 
The algorithm proceeds as follows. While its current search point has fitness at least $n$, it finds the local optimum $z$ using Randomized Local Search (RLS). This takes expected time $O(n\log n)$. From $z$ it jumps to the starting point $z^0=\bar{z}$ of the path $P\big(z,\vec i\, \big)$. The algorithm now follows the path by using again RLS but accepting an offspring if and only if it increases the parent's fitness by exactly $1$ or if the offspring's fitness is $2n+1$. In particular, in this phase the algorithm rejects any search point with fitness between $n$ and $2n$. Since this algorithm needs time $O(n)$ to advance one step on the path, and the path has length $O(n)$, it has expected runtime $O(n^2)$.

For the lower bound we again extend the class of elitist black-box algorithms to a larger class $\A$ that allows to apply Yao's Principle. After an algorithm in $\A$ has sampled its first search point, we distinguish two cases. If the search point has fitness at most $n+n/4$, then the algorithm may access the position of the global optimum (and thus, terminate in one more step). If the first search point has fitness larger than $n+n/4$, then the algorithm may access the position of the local optimum $z$. Moreover, it may access a counter that tells it how many steps it has performed so far. Apart from that, it may only access (one of) the best search point(s) it has found so far, and its fitness. Then $\A$ is the set of all algorithms that can be implemented with this additional information. In this way, every randomized algorithm in $\A$ is a convex combination of deterministic ones, so that we can apply Yao's Principle. So let $A \in \A$ be a deterministic algorithm, and consider the uniform distribution on $\F$.

If the first search point has fitness $2n+1$ or at most $n+n/4$, then $A$ is done after one query or it can terminate in at most one additional step, respectively. However, by the Chernoff bound these two events happen only with probability $e^{-\Omega(n)}$, so from now on we assume that the first search point has fitness larger than $n+n/4$. Observe that by the accessible information the algorithm can determine the \onemax value $\OM_{z}(x)$ for all $x\in \{0,1\}^n$. In particular, for every search point of larger fitness except for $z^{\ell}$ the algorithm can predict the fitness value without querying it. On the other hand, if it queries a search point of lower fitness, then it is not allowed to keep its fitness value. Thus $A$ cannot obtain additional information about $f$ except by querying the optimum $z^{\ell}$. Since $\F$ was chosen uniformly at random, all search points in distance $\ell = n/4$ from $z^0=\bar{z}$ have the same probability to be the global optimum. Hence, $A$ needs in expectation at least $\binom{n}{n/4}/2 = 2^{\Omega(n)}$ queries to find the optimum.
\end{proof}

\begin{remark}\label{rem:comparison}
A similar statement as the one in Theorem~\ref{thm:example2} holds also for ranking-based algorithms if we slightly increase the memory of the algorithms regarded. 
Indeed, there exists a unary unbiased (2+1) memory-restricted ranking-based algorithm optimizing $\F$ in expected $O(n^2)$ function evaluations. Regard, e.g., the algorithm that maintains throughout the second phase a search point $x^1$ of fitness $n+1$ and that accepts an offspring $y$ of $z^j$ if and only if the fitness of $y$ is larger than that of $z^j$ but smaller than that of $x^1$ (in which case $y = z^{j+1}$). Then $z^{\ell}$ is sampled (but not accepted into the population, see Remark~\ref{rem:comparison2}) after $O(n^2)$ steps.

On the other hand, the (2+1) elitist black-box complexity is still exponential, since with high probability the first two search points the algorithm samples have fitness $n/2+o(n)$. 
\end{remark}

\begin{remark}\label{rem:comparison2}
As indicated in Remark~\ref{rem:comparison} it can make a crucial difference for (non-elitist) black-box algorithms if we only require them to \emph{sample} an optimum or whether we require the algorithm to \emph{accept} it into the population. For example, the algorithm described in Remark~\ref{rem:comparison} does not accept the optimum when finding it. 
\end{remark}


\section{Combining Unbiased and Elitist Black-Box Models}
\label{sec:jump}

In this section we demonstrate that apart from providing more realistic lower bounds for some function classes, the elitist black-box model is also an interesting counterpart to existing black-box models. Indeed, we show that some of the unrealistically low black-box complexities of the unbiased black-box model proposed in~\citep{LehreW12} disappear when elitist selection is required. 

More specifically, we regard the unary unbiased (1+1) elitist black-box complexity\footnote{That is, the complexity with respect to all (1+1) elitist black-box algorithms for which the sampling distributions in line~\ref{line:mut} of Algorithm~\ref{alg:elitist} are unbiased in the sense of Lehre and Witt~\citep{LehreW12}.} of \jump functions, which (in line with~\citep{DoerrDK14jump}) we define in the following way. For a parameter $k$ the function $\jump_k$ assigns to each bit string $x$ the function value $\jump_k(x) = \onemax(x)$ if $\onemax(x) \in \{0\} \cup \{k+1,\ldots,n-k-1\} \cup \{n\}$ and $\jump_k(x) = 0$ otherwise.
Despite the fact that all common search heuristics need $\Omega(n^{k+1})$ fitness evaluations to optimize this function, the unary unbiased black-box complexity of these functions are surprisingly low, see Table~1 for a summary of results presented in~\citep{DoerrDK14jump} (\citep{DoerrKW11} for $k=1$). Interestingly, even for extreme jump functions in which only the fitness value $n/2$ is visible and all other \onemax values are replaced by zero, polynomial-time unary unbiased black-box algorithms exist. It is thus interesting to see that the situation changes dramatically when the algorithms are required to be elitist, as the following theorem shows.

\begin{table*}[t]
\label{tab:jump}
\begin{center}  
\begin{tabular}{c|c|c|c}
  Model & range of $k$ &  unary unbiased & elitist unary unbiased \\
  \hline
Constant Jump &  $1\leq k = \Theta(1)$ & $\Theta(n \log n)$ & $\Theta(n^{k+1})$ \\
  \hline
Short Jump & $k = O(n^{1/2-\eps})$ & $\Theta(n \log n)$ & $\Theta(\binom{n}{k+1}) = \Omega\left((n/k)^k\right)$ \\
\hline
Long Jump & $k = (1/2 - \varepsilon)n$ & $O(n^2)$  & $\Theta(\binom{n}{k+1}) = 2^{\Theta(n)}$  
\\
\hline
Extreme Jump & $k = n/2 - 1$ & $O(n^{9/2})$ & $\Theta(2^n/\sqrt{n})$ \\
\end{tabular}
\end{center}
\caption{Comparison of the unary unbiased black-box complexities of $\jump_k$ with the respective (Las Vegas and Monte Carlo) elitist ones for different regimes of $k$. 
}
\end{table*}


\begin{theorem}
\label{thm:jump}\label{THM:JUMP}
For $k=0$ the (Las Vegas and Monte Carlo) unary unbiased (1+1) elitist black-box complexity of the jump function $\jump_k$ is $\Theta(n \log n)$. 
For all $1 \leq k \leq n/2-1$ it is $\Theta(\binom{n}{k+1})$. 
In particular, for $k= \omega(1)$ the black-box complexity 
is superpolynomial in $n$ and for $k= \Omega(n)$ it is exponential.
\end{theorem}
\begin{proof}[Proof of the Upper Bound in Theorem~\ref{thm:jump}]
For any constant $k$ the upper bound is achieved by the simple \oea~\citep{DrosteJW02}. For general $k$, consider the algorithm $A$ that produces an offspring as follows. With probability $1/3$ the offspring is a search point uniformly at random from $\{0,1\}^n$, with probability $1/3$ the algorithm flips exactly one bit (uniformly at random), and with probability $1/3$ it flips exactly $k+1$ bits (also uniformly at random). The offspring is accepted if its fitness is at least the fitness of the current search point. This algorithm finds a point of positive fitness in expected time $O(\sqrt{n})$ since it produces random search points with probability $1/3$ and each such uniform sample has \onemax value $n/2$ with probability $\Theta(1/\sqrt{n})$. Then with high probability it increases the fitness to $n-k-1$ in at most $O(n\log n)$ steps by one-bit flips (and possibly $(k+1)$-bit flips). Afterwards, since there are $\binom{n}{k+1}$ search points in distance $k+1$, the algorithm needs in expectation at most $3\binom{n}{k+1}$ steps to find the optimum. This proves the upper bound for the Las Vegas complexity, which in turn implies the upper bound for the Monte Carlo complexity. 
\end{proof}

It remains to prove the lower bound. For $k=0$ it follows from~\citep[Theorem 6]{LehreW12}. For general $k$, as an intermediate step, we show the following general result.

\begin{theorem}
\label{thm:distance}
Assume that $f$ is a function with a unique global maximum $\xopt$, and assume further that a unary unbiased (1+1) elitist black-box algorithm $A$ is currently at a search point $x \neq \xopt$. Let $0<d \leq n/2$, and let $\dist$ denote the Hamming distance. Let
$$S:=\{x' \neq \xopt\mid \dist\{x' , \xopt\} \leq d \text{ or } \dist\{x' , \xopt\} \geq n-d\}$$
be the distance-$d$ neighborhood of $\xopt$ and of its bitwise complement $\overline{\xopt}$. If all search point in $S$ have fitness less than $f(x)$, then $A$ needs in expectation at least $\binom{n}{d+1}$ additional queries to find the optimum. Moreover, for every $\alpha \geq 0$, the probability that it needs at most $\alpha \binom{n}{d+1} $ additional queries is at most $\alpha$.
\end{theorem}

\begin{proof}
Since $A$ is elitist, it can never accept a point in $S$. Therefore, in every subsequent step before finding the optimum, it will be in some search point $y$ with distance $d' \in [d+1, n-d-1]$ from the optimum. If an unbiased mutation has some probability $p$ to produce $\xopt$ from $y$, then every other search point in distance $d'$ has also probability $p$ to be the offspring. In particular, since there are $\binom{n}{d'}$ such points, $p\binom{n}{d'}$ equals the probability that the offspring has distance $d'$ of $y$, which is at most $1$. Hence, at any point the probability to sample the optimum in the next step is at most $1/\binom{n}{d'} \leq 1/\binom{n}{d+1}$. Therefore, $A$ needs in expectation at least $\binom{n}{d+1}$ steps to find the optimum. Moreover, by the union bound the probability that $A$ needs less than $\alpha \binom{n}{d+1}$ steps is at most $\alpha$.
\end{proof}


\begin{proof}[Proof for the Lower Bound in Theorem~\ref{thm:jump}]
Assume for simplicity that $n$ is a power of $2$. In order to apply Yao's Principle (Lemma~\ref{lem:Yao}), we allow the algorithm to remember its complete search history. We then regard a deterministic algorithm $A$ on a uniformly chosen $\jump_k$ function. That is, the target string $z$ of the \onemax function underlying the $\jump_k$ function is chosen from $\{0,1\}^n$ uniformly at random.

It is known that we can partition the hypercube $\{0,1\}^n$ into $2^n/n$ sets $S_1, \ldots, S_{2^n/n}$ of size~$n$ each such that for each $i \in [2^n/n]$ the pairwise distance between any two points in $S_i$ is exactly $n/2$ (e.g., the cosets of the Hadamard code as in~\cite{AroraB09}). 

Let $i$ be the index of the set containing the target string~$z$, i.e., $z \in S_i:=\{s^1, \ldots, s^n\}$. Regardless of the jump-size~$k$, each search point in $S_i$ has positive fitness. Indeed, for each $j$ either we have $s^j=z$ (in which case the fitness of $s^j$ equals $n$) or the distance and thus the fitness of $s^j$ to $z$ equals $n/2$. Since $z$ is chosen uniformly at random, the probability that the first search point of set $S_i$ that the algorithm $A$ queries is the target string~$z$ equals $1/n$. This shows that with probability at least $1-1/n$ the first search point with positive fitness that $A$ queries is not the optimum. But then Theorem~\ref{thm:distance} tells us that the algorithm needs at least $\binom{n}{k+1}$ additional steps in expectation, and at least $\alpha\binom{n}{k+1}$ with probability at least $\alpha$. This proves the claim.
\hspace*{\fill}
\end{proof}

\section{Conclusions}
\label{sec:conclusions}
We have introduced elitist black-box complexity as a tool to analyze the performance of search heuristics with elitist selection rules. Several examples provide evidence that the elitist black-box complexities can give a much more realistic estimation of the expected runtime of typical search heuristics. We have also seen that some unrealistically low black-box complexities in the unbiased model disappear when elitist selection is enforced. 

We have also introduced the concept of Monte Carlo black-box complexities and have brought to the attention of the community the fact that these can be significantly lower than the previously regarded Las Vegas complexities. In addition, it can also be significantly easier to derive bounds for the Monte Carlo black-box complexities, see~\citep{DoerrL15OM}. Both complexity notions correspond to runtime analysis statements often seen in the evolutionary computation literature and should thus co-exist in black-box complexity research.

While we regard in this work toy-problems, it would be interesting to analyze the influence of elitist selection on the performance of algorithms in more challenging optimization problems. Our findings enliven the question for which problems non-elitist selection like tournament or so-called fitness-dependent selection can be beneficial, initial findings for which can be found in~\citep{OlivetoZ15, FriedrichOSW09}. Negative examples are presented in~\citep{OlivetoW14,HappJKN08,NeumannOW09}. 

\subsection*{Acknowledgments}
This research benefited from the support 
of the ``FMJH Program Gaspard Monge in optimization and operation research'', 
and from the support to this program from EDF.
}

\small

\bibliographystyle{apalike}

\end{document}